\newif\ifisarxiv
\isarxivtrue

\ifisarxiv
\documentclass[12pt]{sty/colt2019/colt2018-arxiv}
\else
\documentclass[anon,12pt]{sty/colt2019/colt2019}
\fi
\usepackage{times}
\usepackage{color}
\usepackage{cancel}
\usepackage{wrapfig}
\usepackage{caption}
\usepackage{algorithm}
\usepackage{algorithmic}
\def\Vol{{\textnormal{R-DPP}}}
\def\DPP{{\mathrm{DPP}}}
\def\VS{{\mathrm{VS}}}
\def\nnz{{\mathrm{nnz}}}
\def\poly{{\mathrm{poly}}}
\def\simiid{\overset{\textnormal{\fontsize{6}{6}\selectfont
i.i.d.}}{\sim}}
\def\tinydots{\textnormal{\fontsize{6}{6}\selectfont \dots}}

\def\sigmat{\tilde{\sigma}}

\def\xbt{\widetilde{\x}}
\def\vbh{\widehat{\v}}

\def\val {{\mathrm{val}}}

\def\Xt{\widetilde{X}}
\def\Kt{\widetilde{K}}

\def\Lb{\mathbf{L}}

\def\Y{\mathbf Y}
\def\R{\mathbf R}

\ifx\BlackBox\undefined
\newcommand{\BlackBox}{\rule{1.5ex}{1.5ex}}  
\fi

\DeclareMathOperator*{\argmin}{\mathop{\mathrm{argmin}}}

\def\x{\mathbf x}

\def\v{\mathbf v}

\def\zero{\mathbf 0}
\def\one{\mathbf 1}

\def\st{\tilde{s}}
\def\lt{\tilde{l}}

\def\X{\mathbf X}

\def\B{\mathbf B}
\def\A{\mathbf A}
\def\C{\mathbf C}

\def\V{\mathbf V}

\def\St{\widetilde{S}}

\def\I{\mathbf I}

\def\A{\mathbf A}
\def\P{\mathbf P}

\def\Xt{\widetilde{\mathbf{X}}}

\def\Ot{\widetilde{O}}

\def\E{\mathbb E}

\def\R{\mathbb R} 
\def\tr{\mathrm{tr}}

\def\rank{\mathrm{rank}}

\newcommand{\defeq}{\stackrel{\textit{\tiny{def}}}{=}}

\let\origtop\top
\renewcommand\top{{\scriptscriptstyle{\origtop}}} 

\definecolor{silver}{cmyk}{0,0,0,0.3}
\definecolor{yellow}{cmyk}{0,0,0.9,0.0}
\definecolor{reddishyellow}{cmyk}{0,0.22,1.0,0.0}
\definecolor{black}{cmyk}{0,0,0.0,1.0}
\definecolor{darkYellow}{cmyk}{0.2,0.4,1.0,0}
\definecolor{darkSilver}{cmyk}{0,0,0,0.1}
\definecolor{grey}{cmyk}{0,0,0,0.5}
\definecolor{darkgreen}{cmyk}{0.6,0,0.8,0}
\newcommand{\Red}[1]{{\color{red}  {#1}}}

\newcommand{\Green}[1]{{\color{darkgreen}  {#1}}}
\newcommand{\Blue}[1]{\color{blue}{#1}\color{black}}
\newcommand{\Brown}[1]{{\color{brown}{#1}\color{black}}}
\newcommand{\white}[1]{{\textcolor{white}{#1}}}

\ifx\proof\undefined
\newenvironment{proof}{\par\noindent{\bf Proof\ }}{\hfill\BlackBox\\[2mm]}
\fi

\ifx\theorem\undefined
\newtheorem{theorem}{Theorem}
\fi

\ifx\example\undefined
\newtheorem{example}{Example}
\fi

\ifx\property\undefined
\newtheorem{property}{Property}
\fi

\ifx\lemma\undefined
\newtheorem{lemma}[theorem]{Lemma}
\fi

\ifx\proposition\undefined
\newtheorem{proposition}[theorem]{Proposition}
\fi

\ifx\remark\undefined
\newtheorem{remark}[theorem]{Remark}
\fi

\ifx\corollary\undefined
\newtheorem{corollary}[theorem]{Corollary}
\fi

\ifx\definition\undefined
\newtheorem{definition}{Definition}
\fi

\ifx\conjecture\undefined
\newtheorem{conjecture}[theorem]{Conjecture}
\fi

\ifx\axiom\undefined
\newtheorem{axiom}[theorem]{Axiom}
\fi

\ifx\claim\undefined
\newtheorem{claim}[theorem]{Claim}
\fi

\ifx\assumption\undefined
\newtheorem{assumption}[theorem]{Assumption}
\fi

\title[Fast determinantal point processes]
{Fast determinantal point processes\\
  via distortion-free intermediate sampling}
\coltauthor{%
 \Name{Micha{\l } Derezi\'{n}ski} \Email{mderezin@berkeley.edu}\\
 \addr  Department of Statistics, UC Berkeley
}

\begin{document}
\maketitle

\begin{abstract}
Given a fixed $n\times d$ matrix $\X$, where $n\gg d$, we study the
complexity of sampling from a 
distribution over all subsets of rows where the probability of a subset is
proportional to the squared volume of the parallelepiped spanned by the rows (a.k.a.~a determinantal
point process). In this task, it is important
to minimize the preprocessing cost of the procedure (performed once)
as well as the 
sampling cost (performed repeatedly). To that end, we propose
a new determinantal point process algorithm which has the following
two properties, both of which are novel:
(1) a preprocessing step which runs in time
$O\big(\text{number-of-non-zeros}(\X)\cdot\log n\big)+\poly(d)$, and (2)
a sampling step which runs in $\poly(d)$ time, independent of the
number of rows $n$. We achieve this by introducing a new
\textit{regularized} determinantal point process (R-DPP), which serves
as an intermediate distribution in the sampling procedure by reducing
the number of rows from $n$ to $\poly(d)$. Crucially, this
intermediate distribution does not distort the probabilities of the
target sample. Our key novelty in defining the R-DPP 
is the use of a Poisson random variable for controlling the
probabilities of different subset sizes, leading to new determinantal
formulas such as the normalization constant for this distribution. Our algorithm
has applications in many diverse areas where determinantal point
processes have been used, such as 
machine learning, stochastic optimization, data summarization and
low-rank matrix reconstruction.
\end{abstract}




\section{Introduction}\label{s:intro}
Determinantal point processes (DPP) form a family of distributions
sampling diverse subsets of points from a given domain, where the
diversity is measured by the squared volume of the parallelepiped
spanned by the points in some predefined space. 
DPPs have found applications 
in a variety of fields such as physics \citep{dpp-physics}, statistics
\citep{dpp-stats}, machine learning \citep{dpp-ml}, computational geometry
\citep{pca-volume-sampling}, graph theory \citep{spanning-trees} and
others. The applications include:
\vspace{-1mm}
\begin{enumerate}
  \item \textit{Data summarization and diverse recommendation}
    \citep[e.g.,][]{dpp-salient-threads,dpp-video,dpp-summarization}:
    selecting a representative sample of items, e.g.~documents in a corpus,
    frames in a video or products in an online store.
    \vspace{-1.5mm}
  \item \textit{Row-based low-rank matrix reconstruction}
    \citep[e.g.,][]{pca-volume-sampling,more-efficient-volume-sampling}:
    determinantal sampling is the optimal way of selecting
    few rows of a matrix that preserve its low-rank approximation.
    \vspace{-1.5mm}
  \item \textit{Stochastic optimization and Monte Carlo sampling}
    \citep[e.g.,][]{dpp-minibatch,dpp-mcmc}: DPP sampling has been used 
    to reduce the variance inherent in i.i.d.~sampling and improve
    convergence. 
  \end{enumerate}
  \vspace{-2mm}
  
   Let $\X\in\R^{n\times d}$ be a tall and thin matrix,
   i.e.~$n\gg d$. Determinantal point process $\DPP(\X)$ is defined as a distribution
over  all $2^n$ subsets $S\subseteq \{1..n\}$ such that
\begin{align}
\Pr(S) \ =\  \frac{\det(\X_S\X_S^\top)}{\det(\I+\X\X^\top)},\label{eq:dpp}
\end{align}
where $\X_S\in\R^{|S|\times d}$ denotes the submatrix of $\X$ containing
rows indexed by $S$.
DPP algorithms are typically divided into a preprocessing step,
which needs to be performed once per given matrix $\X$, and a sampling
step, which happens each time we wish to sample set
$S\sim\DPP(\X)$. In this paper, we improve on the best known time
complexity of both steps by demonstrating:
\begin{enumerate}
\item the first DPP algorithm with \textit{input sparsity time} preprocessing:
$\nnz(\X)\log n + \poly(d)$;
\item the first exact DPP algorithm s.t.~sampling
  takes $\poly(d)$ time, independent of $n$.
\end{enumerate}
Here, $\nnz(\X)$ denotes the number of non-zero entries. Before this
work, the best known DPP algorithms (see Section \ref{s:background}
and references therein) had preprocessing times $\Theta(nd^2)$, or
$\Omega(\nnz(\X)k^2)$ if we allow conditioning on a fixed subset size
$|S|=k$ (which can be as large as $d$), and
sampling times at least $\Omega(n\,|S|)$.
Our DPP algorithm is based on the following general recipe for
sampling from some ``target'' joint distribution:
\ $i)$ generate a larger sample of $\poly(d)$ rows
$\sigma=(\sigma_1,\dots,\sigma_k)\in\{1..n\}^k$ from an
``intermediate'' distribution; \ $ii)$ downsample to a smaller subset using the
``target'' distribution. Crucially, the intermediate sampling is not
allowed to distort the target distribution:
\begin{align*}
\textbf{Goal:}\qquad\sigma\,\overset{i)}{\sim}\,\text{intermediate}\!\overbrace{(\X)}^{n\times
  d}\ \ \text{\small and}\ \  
  S\,\overset{ii)}{\sim}\,\text{target}\hspace{-3mm}\overbrace{(\X_\sigma)}^{\poly(d)\times d}\quad\ \Longrightarrow\quad\quad
\overbrace{\{\sigma_i\}_{i\in S}}^{\sigma_S}\,\sim\,\overbrace{\text{target}(\X)}^{\DPP(\X)}.
\end{align*}
A simplified version of this approach was recently suggested by
\cite{leveraged-volume-sampling}. They sample from a different but
related family of determinantal distributions called \textit{size $k\geq d$
  volume sampling}, which has the unique property that it can play the
role of both the ``intermediate'' and the ``target'' distribution.
DPPs on the other hand do not have that property and no
candidate for an intermediate distribution was previously known.
To that end, we develop a new family of joint distributions:
\textit{regularized} determinantal point  processes (R-DPP). For
a p.s.d.~matrix $\A\in\R^{d\times d}$ and 
a Poisson mean parameter $\Blue{r}>0$ an $\Vol^{\Blue{r}}(\X,\A)$ samples
over all sequences $\sigma\in\bigcup_{k=0}^\infty\{1..n\}^k$ so that
\vspace{-1mm}
\begin{align*}
  \Pr(\sigma) \ \propto \ \det\!\big(\A +
  \X_{\sigma}^\top\X_{\sigma}\big)\ 
  \frac{\Blue{r}^k\text{e}^{-\Blue{r}}}{k!},\qquad\text{where $k$ is
  the length of $\sigma$}.
\end{align*}
Rescaling with
Poisson probabilities $\frac{\Blue{r}^k\text{e}^{-\Blue{r}}}{k!}$ is
\textit{essential} for the normalization constant of this distribution
to have a closed form (a key part of our analysis). We obtain it
via the following new determinantal formula: if
$\sigma = (\sigma_1,\dots,\sigma_K)$ is a sequence of $K$
i.i.d.~samples from $\{1..n\}$, then
\begin{align*}
\text{for}\quad
  K\sim\text{Poisson}(\Blue{r}),\qquad
  \E\Big[\det\!\big(\A+\X_{\sigma}^\top\X_{\sigma}\big)\Big]
  = \det\!\Big(\A + \E\big[\X_{\sigma}^\top\X_{\sigma}\big]\Big),
\end{align*}
where the expectations are over the random variable $K$ as well as the
i.i.d.~samples.
Similar formulas have been known for fixed length $K$ but only in the
unregularized case where $\A=\zero$ 
\citep[see][]{correcting-bias}. Our result
shows that regularization can be introduced by randomizing the
sequence length with Poisson distribution.
In Section~\ref{s:r-dpp} we use this formula to show a number of
connections between R-DPPs and DPPs. For example, the latter can be
obtained from the former when the Poisson mean parameter $\Blue{r}$
converges to 0 along with regularization $\A$ set to $\frac{\Blue{r}}{n}\,\I$: 
\begin{align*}
  \Vol^{\Blue{r}}\!\Big(\,\X,\
  \frac{\Blue{r}}{n}
  \,\I\,\Big)   \
  &\overset{\Blue{r}\rightarrow 0}{\longrightarrow}\ \DPP(\X).
\end{align*}
This means that the family of R-DPPs contains DPPs in its
closure and can thus be viewed as an extension. We also show that DPPs are preserved under subsampling with an
R-DPP:
\begin{align*}
\text{if}\quad
  \sigma\sim\Vol^{\Blue{r}}(\X,\,\I)\quad\text{and}\quad
  S\sim\DPP(\X_\sigma),\quad\text{then}\quad
  \sigma_S\sim\DPP\Big(\text{\scriptsize$\sqrt{\frac{\Blue{r}}{n}}$}\,\X\Big).
\end{align*}
This suggests that R-DPP is a good candidate for an ``intermediate''
distribution when sampling a DPP. To make sampling from R-DPPs
efficient, we further generalize them so that the marginal row
probabilities can be reweighted with an arbitrary
i.i.d.~distribution. On a very high level, our strategy is to
show that when the length of sequence $\sigma$ is sufficiently large in expectation
(but still independent of $n$) and the R-DPP is reweighted by ridge leverage
scores, then it becomes very close to i.i.d.~sampling, so we can use that as a
proposal distribution for a rejection sampling scheme.


In the following section we give some
background and related work on DPP algorithms, then in Section
\ref{s:main} we present our main algorithm (Algorithm \ref{alg:main})
and the associated result (Theorem \ref{t:main}), along with an
example application in low-rank matrix reconstruction.  Section
\ref{s:r-dpp} introduces R-DPPs along with their basic properties and the
remaining Sections \ref{s:correctness} and \ref{s:efficiency} are
devoted to proving Theorem \ref{t:main}. 


\section{Background and related work}\label{s:background}

Several settings have been used in the literature for studying the
complexity of DPP algorithms. Below we review those which are slightly
different than the one we presented in Section \ref{s:intro}, but are
still relevant to our discussion \citep[see][for details]{dpp-ml}.\\[3mm]
\textbf{$\Lb$-ensemble} \ We defined a DPP in terms
of a matrix $\X\in\R^{n\times d}$, where each element $i\in\{1..n\}$ is described
by a row vector $\x_i^\top\in\R^d$
\citep[suggested by][]{pca-volume-sampling,structured-dpp}. 
An equivalent parameterization can be defined 
in terms of the so-called \textit{ensemble} matrix
$\Lb=\X\X^\top\in\R^{n\times n}$, where the $(i,j)$th entry represents
the dot product $\x_i^\top\x_j$. In this case, the probability 
\eqref{eq:dpp} of a subset $S\subseteq \{1..n\}$ can be written as
$\Pr(S)=\frac{\det(\Lb_{S,S})}{\det(\I+\Lb)}$, where
$\Lb_{S,S}\in\R^{|S|\times|S|}$ denotes the submatrix of $\Lb$ with
both row and column entries indexed by $S$. Naturally, one
representation can be converted to the other in preprocessing, so our
results are still useful in the $\Lb$-ensemble case when
$\text{rank}(\Lb)\ll n$. However, since matrix $\Lb$ is much larger
than $\X$, the preprocessing cost may increase.\\[3mm]
\textbf{k-DPP} \ In some practical applications, when a subset of particular
size is desired, a DPP can be restricted only to subsets
$S$ such that $|S|=k$ for some $k\in\{1..d\}$, and referred to as a k-DPP
\citep{k-dpp}. This is equivalent to sampling a set $S$ from a standard
DPP, but accepting the sample only if $|S|=k$. This distribution is also
sometimes called \textit{size $k\, \Blue{\leq}\, d$  volume sampling}
\citep{pca-volume-sampling},
not to be confused with size $k \,\Red{\geq}\, d$ volume sampling
mentioned in Section \ref{s:intro}. The special case of $k=d$ will be further discussed in Section \ref{s:classic}.
An alternative way of controlling the subset size is by rescaling
matrix $\X$ with some $\alpha$ so that the  expected subset size for 
$S\sim\DPP(\alpha\X)$, i.e.~$\E\big[|S|\big]$ matches a
desired value (see Section \ref{s:main} for more details). A
restriction to k-DPP can lead to faster sampling algorithms (when $k$
is small), as discussed in the following sections.
\vspace{3mm}

A classical DPP algorithm introduced by \cite{dpp-independence}
uses the singular value decomposition (SVD)
of either $\Lb$ or $\X$ to produce an exact 
sample $S$ from the DPP in time $O(n\,|S|^2)$. However, this runtime
does not  include the cost of SVD, considered as a preprocessing
step, which takes $O(n^3)$ and $O(nd^2)$ for $\Lb$ and $\X$,
respectively. Since then, a number of methods were proposed to improve
on this basic approach. We survey these techniques in the following
two sections, and present a runtime comparison in Tables
\ref{tab:preprocessing} and \ref{tab:sampling} (we omit the big-O
notation in the tables). 
\vspace{3mm}

\hspace{-3mm}
\begin{minipage}{0.45\textwidth}
  \centering
\begin{tabular}{r||l}
&preprocessing\\
  \hline\hline
SVD  &$nd^2$\\
sketching 
  &$\nnz(\X)k^2+nk^4$\\
MCMC 
  &$nd\cdot \poly(k)$\\
  \hline
\textbf{this paper}
  &$\nnz(\X) + d^3k^2$
\end{tabular}
\captionof{table}{Preprocessing costs for approximate k-DPPs from $\X$
(omitting log terms), compared to our DPP algorithm.}
\label{tab:preprocessing}
\end{minipage}
\hspace{5mm}
\begin{minipage}{0.45\textwidth}
  \centering
\begin{tabular}{r||l}
& sampling\\
  \hline\hline
bottom-up  
  &$n k^2$\\
  i.i.d.+top-down
  &$nk+k^4$\\
  MCMC 
  & $n\cdot \poly(k)$\\
  \hline
\textbf{this paper}
  &$d^3k$
\end{tabular}
\captionof{table}{Sampling cost (after preprocessing) for DPP/k-DPP methods,
 compared to our DPP algorithm ($k=|S|\leq d\leq n$).}
\label{tab:sampling}
\end{minipage}

\subsection{Previous approximate preprocessing techniques}
Approximate preprocessing methods were studied primarily for k-DPPs
(i.e. $|S|=k$) rather than for standard 
DPPs, because bounding the subset size
makes volume approximations easier to control.
\cite{efficient-volume-sampling} and \cite{dpp-salient-threads} suggested to 
use volume-preserving sketching of \cite{volume-sketching} to reduce
the dimension $d$ of matrix $\X$ to $r=\Ot(k^2)$, which
allows approximate sampling from a k-DPP.
Here, the cost of
sketching is $\Ot(\nnz(\X)k^2)$, and it is followed by
computing SVD of an $n\times r$ matrix in time
$\Ot(nk^4)$. Also, \cite{rayleigh-mcmc} proposed to use a fast-mixing
MCMC algorithm whose stationary distribution is a k-DPP, where
preprocessing cost is
$O(nd\cdot\poly(k))$ for matrix $\X$, and
$O(n\cdot\poly(k))$ for matrix $\Lb$ (sampling time
is similar). Other approximation techniques such as Nystrom
\citep{dpp-nystrom} and coresets \citep{dpp-coreset} yield approximate
DPP distributions however their accuracy is data-dependent.
Table \ref{tab:preprocessing} compares the preprocessing costs for the
approximate k-DPP methods offering data-independent accuracy
guarantees with that of our DPP algorithm. Note that our approach is
specifically designed for 
sampling from a full DPP, not a k-DPP. Also, unlike these
approximate methods our algorithm samples exactly from $\DPP(\rho\X)$ for some
$\rho\approx 1$, which is a much more precise guarantee.


\subsection{Sampling a DPP as a mixture of volume samples}
\label{s:classic}
Any DPP is a mixture of so-called \textit{elementary DPPs}
\citep[see][]{dpp-independence,dpp-ml}. These elementary DPPs have been
independently studied in the context of 
\textit{volume sampling}
\citep{avron-boutsidis13}.
For $\X\in\R^{n\times d}$, volume sampling is given by
\begin{align*}
  S\sim\VS(\X)\!:\quad  \Pr(S) =
  \frac{\det(\X_S)^2}{\det(\X^\top\X)}\quad\text{for 
  $S$ s.t.~} |S|=d.
\end{align*}
The mixture decomposition shown by \cite{dpp-independence} implies that 
DPP sampling can be divided into two steps: first sample one
element from the mixture, then 
generate a sample from that elementary DPP.
Specifically, consider the eigendecompositions
  $\Lb = \sum_{i=1}^d\lambda_i \v_i\v_i^\top$ and $\X^\top\X =
  \sum_{i=1}^d\lambda_i\vbh_i\vbh_i^\top$. For convenience, let us put the
eigenvectors of $\Lb$ into a matrix $\V = 
[\v_1,\dots,\v_d]\in\R^{n\times d}$. 
Then $S\sim\DPP(\X)$ can be produced by sampling a subset $T$ of
eigenvector indices (step 1) and then performing volume sampling w.r.t.~the
$n\times |T|$ matrix  constructed from those vectors (step 2): 
\begin{align*}
S\sim\VS\big(\V_{\!*,T}\big),\qquad
  \text{where for each $i\in\{1..d\}$, independently, }\quad\Pr(i\in T) = \frac{\lambda_i}{1+\lambda_i}.
\end{align*}

Here, we used vectors from the decomposition of $\Lb$, but this can be
easily obtained from the decomposition of $\X^\top\X$ during
preprocessing, since $\v_i =   \frac{1}{\sqrt{\lambda_i}}\X\vbh_i$. So
given the eigendecomposition we can sample the set $T$ easily, and it
remains to 
perform the volume sampling step.

\begin{wrapfigure}{r}{0.45\textwidth}
  \vspace{-7mm}
  \centering
\begin{minipage}{0.45\textwidth}
\floatname{algorithm}{\small Algorithm}
\begin{algorithm}[H] 
 \caption{\small Bottom-up volume sampling}\label{alg:bottom-up}
  \begin{algorithmic}[0]
    \STATE \textbf{input:} $\V\in\R^{n\times k}$, s.t.~$\V^\top\V=\I$
    \STATE \textbf{output:} $S\sim\VS(\V)$
    \STATE \textbf{for }$i=1..k$
    \STATE  \quad Sample $\sigma_i \sim \big(\|\V_{1,*}\|^2,\tinydots,\|\V_{n,*}\|^2\big)$
    \STATE \quad$\V\leftarrow \V\Big(\I -
    \frac{(\V_{\sigma_i,*})^\top\V_{\sigma_i,*}}{\|\V_{\sigma_i,*}\|^2}\Big)$
    \STATE \textbf{end for}
    \RETURN $S = \{\sigma_1,\dots,\sigma_k\}$
  \end{algorithmic}
\end{algorithm}
\end{minipage}
\vspace{-4mm}
\end{wrapfigure}
In Table \ref{tab:sampling} we review the
running times for different approaches of sampling
$S\sim\VS\big(\V_{\!*,T}\big)$, 
compared to a different MCMC-based DPP sampler \citep{rayleigh-mcmc} and our algorithm. 
The classical approach from DPP literature
\citep{dpp-independence} samples ``bottom-up'', adding one point
at a time and at each step projecting the remaining points onto the
subspace orthogonal to that point (see Algorithm
\ref{alg:bottom-up}). Curiously, a diametrically opposed
``top-down'' approach of \cite{unbiased-estimates-journal}, which eliminates
points one at a time instead of adding them, achieves the same
asymptotic runtime with high probability. The volume sampling algorithm of
\cite{leveraged-volume-sampling} further improves on this downsampling
strategy by introducing an intermediate i.i.d.~oversampling step, which is
what inspired our approach. Note that we are the first to apply the
``top-down" algorithms to DPP
sampling (they were previously known only in the context of volume
sampling). Unfortunately, in all of these methods  
sampling time is linear\footnote{%
A different strategy was proposed by
\cite{dpp-coreset}, which uses coreset construction to approximately sample from a DPP in time
independent of $n$, however the sampling accuracy is data-dependent.} 
in $n$ (they have to read the matrix $\V_{*,T}$ for each sampled set)
which may not be acceptable when $n\gg d$ and
we need to perform the sampling repeatedly.
Our algorithm, which uses the mixture
decomposition only as a subroutine, samples in time independent of
$n$.


\section{Main result}
\label{s:main}
As discussed in Section \ref{s:intro}, the high level strategy of our
algorithm is to design an ``intermediate'' sampling distribution, which
reduces the size of the matrix $\X$ from $n\times d$ to
$\poly(d)\times d$ while preserving the ``target''
distribution (here, a DPP). Another key property of the intermediate
distribution is that it has be close to i.i.d., so that we can
implement it efficiently. Thus, our algorithm will use an
i.i.d.~sampling distribution defined by a vector 
$l=(l_1,\dots,l_n)$ of importance weights assigned to each row of
$\X$, and use it as a proposal for rejection sampling from the
intermediate distribution. Overall, the three main components of our
method are:
\vspace{-1mm}
\begin{align*}
\overbrace{(1)\ \text{i.i.d.~sampling}\qquad(2)\ \text{rejection
  sampling}}^{\sigma\,\sim\,\text{intermediate}(\X)}\qquad \overbrace{(3)\
  \text{downsampling}}^{S\,\sim\,\text{target}(\X_\sigma)}.
\end{align*}
In Section \ref{s:r-dpp} we define a regularized determinantal point
process (R-DPP) and use it as the intermediate distribution. For the
i.i.d.~sampling weights we use a particular variant of so-called ridge
leverage scores \citep{ridge-leverage-scores}.
Let $\x_i^\top$ denote the $i$th row of $\X$ and
w.l.o.g.~assume that all rows are non-zero.
\begin{algorithm}[H] 
  \caption{~ Fast determinantal point process sampling}
  \label{alg:main}
  \begin{algorithmic}[1]
    \STATE \textbf{input:} $\X\in\R^{n\times d}$, \
    $\A\in\R^{d\times d}$, \ sampling oracle for $i\sim\tilde{l}=(\tilde{l}_1,\dots,\tilde{l}_n)$\\[1mm]
    \STATE \textbf{repeat}\label{line:rep1}
    \vspace{1mm}
    \STATE \quad sample $K \sim
    \mathrm{Poisson}(q)$,\hspace{2.5cm} for
    $q=\lceil 2d\st\rceil$,\quad $\tilde{s} = \tr\big(\A(\I+\A)^{-1}\big)$\label{line:poisson}
    \vspace{1mm}
    \STATE \quad sample $\sigma_1,\tinydots,\sigma_K\simiid (l_1,\tinydots,l_n)$,
    \hspace{1.1cm} for $l_i=\x_i^\top(\I\!+\!\A)^{-1}\x_i$,\  {\small rejection
      sampling via $\tilde{l}$}\label{line:iid}
    \vspace{1mm}
    \STATE \quad sample $\textit{Acc}\sim\!
    \text{Bernoulli}\Big(\frac{\det(\I+\Xt_\sigma^\top\Xt_\sigma)}  
{C_K\det(\I+\A)}\Big)$,\,
    for $\Xt = 
\Big[\!\sqrt{\!\frac{\st}{l_{i}\!(q-\st)}}\,\x_{i}^\top\Big]$,
\     $C_K=(\frac{q}{q-\st})^{K+d}\text{e}^{-\st}$\label{line:acc}
\vspace{1mm}
    \STATE \textbf{until} $\textit{Acc}=\text{true}$\hfill
    \textcolor{grey}{\small  (if $\textit{Acc}=\text{true}$, then $\sigma_1,\tinydots,\sigma_K$ is distributed as an R-DPP)}\label{line:rep2}
    \RETURN $\sigma_{\St}$,\quad where $\St\sim \DPP\big(\Xt_\sigma\big)$ \label{line:sub}
 \end{algorithmic}
\end{algorithm}
Other than matrix $\X$, the algorithm takes additional inputs: matrix
$\A\approx \X^\top\X$ and a sampling oracle for $\tilde{l}$
which approximates $l$. The inputs $\A$ and
$\tilde{l}$ are computed in the preprocessing step, 
which can be easily performed in time $O(nd^2)$, but standard
sketching techniques can be used to achieve input sparsity time preprocessing.
In line \ref{line:sub} of the algorithm we invoke a different DPP
sampling procedure for a matrix of reduced size. This can be for
example the classical algorithm discussed in Section \ref{s:classic}
with the volume sampling part implemented by Algorithm
\ref{alg:bottom-up} and SVD performed exactly. 
Our main result shows that Algorithm \ref{alg:main} runs in 
time independent of $n$ and samples from a determinantal point
process.
The only trade-off coming from approximate preprocessing is that it
samples from $\DPP(\rho\X)$ instead of $\DPP(\X)$, 
for some $\rho\approx 1$. This rescaling only affects the distribution of
sample size $|\sigma_{\St}|$ (and not the conditional probability
given the size), and it also implies a weaker (but more standard)
approximation guarantee based on \textit{total variation} distance.
\begin{definition}[total variation]
A distribution on a finite domain $\Omega$ with probability mass
function $q$ is an $\epsilon$-approximation of a distribution on $\Omega$ with
probability mass function $p$ if 
\begin{align*}
  \frac12\sum_{x\in\Omega}\big|q(x) - p(x)\big|\leq \epsilon.
\end{align*}
\end{definition}

\begin{theorem}\label{t:main}
Suppose that $\epsilon\in[0,1]$, $C\geq 0$ and $\X\in\R^{n\times d}$, $\A\in\R^{d\times d}$,
  $\tilde{l}=(\tilde{l}_1,\dots,\tilde{l}_n)$ satisfy:
  \begin{align}
    (1-\eta)\X^\top\X&\preceq \A\preceq(1+\eta)\X^\top\X,
\qquad\qquad \text{for }\ \eta= \frac\epsilon{4\bar{s}+C\ln9/\epsilon},\label{eq:cond1}\\
    \frac12\x_i^\top(\I+\A)^{-1}\x_i&\leq\, \lt_i\,
\leq\frac32\x_i^\top(\I+\A)^{-1}\x_i ,\qquad \ \text{for all }\ i\in\{1..n\},\label{eq:cond2}
  \end{align}
where $\bar{s}=\max\{1,\,\E[|S|]\}\leq d$ for $S\sim\DPP(\X)$. The
following are true for Algorithm \ref{alg:main}:
  \begin{enumerate}
  \item It returns $\sigma_{\St}\sim \DPP(\rho\X)$, where
    $|\rho-1|\leq\eta$ and $\big|\E[|\sigma_{\St}|]-\E[|S|]\big|\leq\epsilon$;
    \item There is $C=O(1)$ for which this distribution is an
      $\epsilon$-approximation of $\DPP(\X)$;
  \item The algorithm has time complexity
    $O(d^3\bar{s}\log^2\!1/\delta)$ w.p.~at least $1-\delta$.
    \end{enumerate}
  \end{theorem}
  \begin{proposition}\label{p:preprocessing}
Matrix $\A$ and distribution $\lt$ satisfying
conditions  \eqref{eq:cond1} and \eqref{eq:cond2} can be obtained from matrix
  $\X$ in time
  $O(\nnz(\X)\log n + d^3\log d\cdot (\bar{s}+\log1/\epsilon)^2/\epsilon^2)$. 
\end{proposition}

The proof of Theorem \ref{t:main} is spread out across
Sections \ref{s:correctness}, \ref{s:efficiency} and
the appendices. In
Section \ref{s:r-dpp} and Appendix \ref{a:r-dpp} we define regularized
determinantal point processes and show how they can be used for
sampling DPPs. Then, in 
Section \ref{s:correctness}  and Appendix \ref{a:ineq} we show that
Algorithm \ref{alg:main} indeed samples from $\DPP(\rho\X)$
($\epsilon$-approximation 
is proven in Appendix \ref{a:tv}). Finally, in Section
\ref{s:efficiency} we present the key steps in proving the time
bounds, with the details (including the proof of
Proposition~\ref{p:preprocessing})
given in Appendix \ref{a:time}.

\subsection*{Application: row-based low-rank matrix
  reconstruction}

One application of DPPs aims to find a
small subset $S\subseteq\{1..n\}$ of rows of matrix $\X\in\R^{n\times
  d}$ such that the subspace spanned by those rows captures the full
matrix nearly as well as the best rank-$k$ approximation in terms of
the Frobenius norm $\|\cdot\|_F$. Let
$\P_S=(\X_S)^+\X_S$ be the projection matrix onto the span of vectors
$\{\x_i\}_{i\in S}$. \cite{more-efficient-volume-sampling}
showed that if $S\sim\DPP(\X)$ then for any $k\leq s\leq d$,
\begin{align*}
  \E\Big[\big\|\X - \X\P_S\big\|_F^2\ \big|\ |S|\!=\!s\Big]\leq
  \frac{s+1}{s+1-k}\,\|\X-\X_{(k)}\|_F^2 \quad\text{where}\quad\X_{(k)} = \!\!\!\argmin_{\Y:\,\text{rank}(\Y)=k}\!\!\|\X-\Y\|_F,
\end{align*}
and that for any $s=o(n)$ the bound is tight up to lower order
terms.\footnote{%
Other methods are known for this and related tasks which achieve
near-optimal bounds \cite[e.g., see][]{pca-volume-sampling,near-optimal-columns,optimal-cur}.}
  In particular, for the ratio $\frac{s+1}{s+1-k}$ to become
$1+\epsilon$, we need $s_{k,\epsilon} = k/\epsilon + k-1$ rows sampled
from a DPP. Even though it is most natural to use fixed-size DPPs in
this context, the sample size $|S|$ for a standard DPP is sufficiently
concentrated around its mean to offer near-optimal
guarantees for this task. In fact, as shown by
\cite{dpp-concentration} (see Lemma \ref{l:concent} in Appendix \ref{a:tv}), for
$S\sim\DPP(\X)$, w.p.~$\geq\frac12$ we have $\big|
|S|-\bar{s}\big|\leq c\sqrt{\bar{s}}$ for some absolute constant
$c$, where $\bar{s}=\max\{1,\,\E[|S|]\}$. The
expected sample size is derived as:
\begin{align}
\E\big[|S|\big] =
  \sum_{i=1}^d\frac{\lambda_i}{1+\lambda_i},\quad\text{where}\quad\lambda_1,\dots,\lambda_d\text{
  are the eigenvalues of }\X^\top\X.\label{eq:bars}
\end{align}
Since $\bar{s}$ is monotonic w.r.t.~the eigenvalues we can use a
simple binary search to find a rescaling $\alpha\X$ for
which $\bar{s}-c\sqrt{\bar{s}}\approx s_{k,\epsilon}$. Thus, if
$S\sim\DPP(\alpha\X)$, then for
$\Delta_{k,\epsilon}=[s_{k,\epsilon},s_{k,\epsilon}+\delta_{k,\epsilon}]$,
where
$\delta_{k,\epsilon}=2c\sqrt{\bar{s}}=O\big(\sqrt{s_{k,\epsilon}}\big)$,
we have
\begin{align*}
    \E\Big[\big\|\X - \X\P_S\big\|_F^2\ \big|\ |S|\in\Delta_{k,\epsilon}\Big]
\leq (1+\epsilon)\,
  \|\X-\X_{(k)}\|_F^2\quad\text{and}\quad
  \Pr\!\big(|S|\in\Delta_{k,\epsilon}\big)\geq \frac12. 
\end{align*}
We can obtain the same guarantee if we
replace the eigenvalues of $\X^\top\X$ in \eqref{eq:bars}
with those of matrix $\A$ satisfying condition
\eqref{eq:cond1} with $\eta=\frac1{4s_{k,\epsilon}}$, so by Theorem
\ref{t:main} the total cost of obtaining such a sample would 
be%
\footnote{If we forgo exact DPP sampling, then 
  \textit{projection-cost preserving sketches}
  \citep{projection-cost-preserving} may offer further speed-ups.}
$O(\nnz(\X)\log n + d^3\log d\cdot s_{k,\epsilon}^2)$. Note that we
needed the fact that Algorithm \ref{alg:main} returns $\DPP(\rho\X)$
rather than just an approximation of $\DPP(\X)$. This raises the
following natural question: can our techniques be extended to 
sampling from fixed-size DPPs, so that the optimal sample size $s_{k,\epsilon}$ can
be achieved exactly and in time $O(\nnz(\X)\log n + \poly(d))$? We leave
this as a new direction for future work. 

\section{Regularized determinantal point processes (R-DPP)}
\label{s:r-dpp}
We propose a new family of determinantal sampling distributions which
will be used in Sections \ref{s:correctness} and \ref{s:efficiency} to
prove Theorem \ref{t:main}. The crucial property of this
family is that the determinantal sampling probabilities can be
regularized by adding an arbitrary fixed positive semi-definite
(p.s.d.) matrix inside of the 
determinant, while maintaining many of the natural properties of a
DPP, such as a simple normalization constant. 
This is achieved by controlling the size of the sample with a
Poisson random variable. In the proofs that follow we will use the shorthand
$[n]\defeq \{1..n\}$.

\begin{definition}
Given matrix $\X\in\R^{n\times d}$, distribution $p=(p_1,\dots,p_n)$,
p.s.d.~matrix $\A\in\R^{d\times d}$ and $r>0$, we define
$\Vol_p^r(\X,\A)$ as a distribution over all index sequences
$\tilde{\sigma}\in\bigcup_{k=0}^\infty\{1..n\}^k$, s.t.
  \begin{align}
  \Pr(\sigmat) = \frac{\det(\A +
  \X_{\sigmat}^\top\X_{\sigmat})}{\det\!\big(\A +
  r\,\E_{j\sim p}[\x_{j}\x_{j}^\top]\big)}\
    \frac{r^{k}\text{e}^{-r}}{k!}\,\prod_{i=1}^{k}p_{\sigmat_i},
    \qquad\text{for}\quad\sigmat\in\{1..n\}^k.\label{eq:poisson-prob}
\end{align}
\end{definition}
Whenever $p$ is uniform, we will write $\Vol^r(\X,\A)$.
Of course, we need to establish that this is in fact a valid
distribution, i.e.~that it sums to one. We achieve this by showing a
new  variant of the classical Cauchy-Binet formula (the classical
formula is stated in \eqref{eq:classic-cb} below). 
\begin{lemma}\label{l:cb}
Given $\X\!\in\!\R^{n\times d}$ and p.s.d.~$\A\!\in\!\R^{d\times d}$, if
$\sigma=(\sigma_1,\dots,\sigma_K)\overset{\textnormal{\fontsize{6}{6}\selectfont
    i.i.d.}}{\sim}p=(p_1,\dots,p_n)$, then
\begin{align*}
  \text{for}\quad K\sim\textnormal{Poisson}(r),\qquad
  \E\Big[\det\!\big(\A +\X_\sigma^\top\X_\sigma\big)\Big]
  = \det\!\Big(\A + \E\big[\X_{\sigma}^\top\X_{\sigma}\big]\Big).
\end{align*}
\end{lemma}
\begin{remark}
The classical Cauchy-Binet formula states that for a matrix
  $\X\in\R^{n\times d}$, we have
  \begin{align}
    \sum_{S\subseteq[n]:\,|S|=d}\det(\X_S)^2 = \det(\X^\top\X).\label{eq:classic-cb}
  \end{align}
Probabilistic extensions of the formula previously appeared in the context of
volume sampling \citep{correcting-bias},
and also much earlier in a different context
\citep{expected-generalized-variance}. In these cases $\A=\zero$ and $K$ is fixed.
\end{remark}
\begin{proof}
Previously shown identities for fixed size $K$ do not generalize
naturally to the regularized setting unless randomness in $K$ is
introduced. We start the proof by applying the Cauchy-Binet formula
\eqref{eq:classic-cb} to the term under the expectation. Let
$\A=\B^\top\B$ be any decomposition of $\A$ s.t.~$\B\in\R^{b\times d}$
where $b=\mathrm{rank}(\A)$. To apply the Cauchy-Binet formula, we sum over all $d$-element subsets of
the union of rows of matrices $\B$ and $\X_{\sigma}$:
\begin{align*}
  \det\!\big(\A + \X_\sigma^\top\X_\sigma\big) =
  \sum_{\underset{|S|\geq d-K}{S\subseteq[b]}}\,\sum_{\underset
  {|T|=d-|S|}{T\subseteq[K]:}}\det\!\bigg(\Big[\substack{\!\B_S\\ ~\X_{\sigma_T}}\Big]\bigg)^{\!2}.
\end{align*}
Applying the law of total expectation w.r.t.~the Poisson variable $K$,
we obtain
\begin{align*}
  \E\Big[\det\!\big(\A+\X_\sigma^\top\X_\sigma\big)\Big]
  &=\sum_{k=0}^\infty \frac{r^k\text{e}^{-r}}{k!}
\sum_{\underset{|S|\geq d-k}{S\subseteq[b]}}\,\sum_{\underset {|T|=d-|S|}{T\subseteq[k]:}}
    \E\bigg[
\det\!\bigg(\Big[\substack{\!\B_S\\ ~\X_{\sigma_T}}\Big]\bigg)^{\!2}
    \,\big|\,K\!=\!k\bigg]
\\
  &\overset{(a)}{=}\sum_{S\subseteq[b]}\,\sum_{k=d-|S|}^\infty
    \frac{r^k\text{e}^{-r}}{k!}\,{k\choose
    d\!-\!|S|}\E\bigg[
\det\!\bigg(\Big[\substack{\,\B_S\\ \,\X_{\sigma}}\Big]\bigg)^{\!2}
    \,\big|\,K\!=\!d\!-\!|S|\bigg]\\
  &\overset{(b)}{=}\sum_{S\subseteq[b]}(d\!-\!|S|)!\!\!\sum_{\underset {|T|=d-|S|}{T\subseteq[n]:}}
    \det\!\bigg(\Big[\substack{\,\B_S\\ \,\X_{T}}\Big]\bigg)^{\!2}
    \bigg(\prod_{i\in
    T}p_i\bigg)
    \sum_{k=d-|S|}^\infty \frac{r^k\text{e}^{-r}}{k!}\,{k\choose
    d\!-\!|S|}\\
  &\overset{(c)}{=}\sum_{S\subseteq[b]}\sum_{\underset
    {|T|=d-|S|}{T\subseteq[n]:}}
    \det\!\Bigg(\!\begin{bmatrix}
      \!\!\B_S\\
      \,\big[\sqrt{r p_i}\,\x_i^\top\big]_{i\in {T}}
    \end{bmatrix}\!\Bigg)^{\!2}
    \underbrace{\sum_{k=d-|S|}^\infty
    \frac{r^{k-d+|S|}\text{e}^{-r}}{(k\!-\!d\!+\!|S|)!}}_{1}\\[-3mm]
  &\overset{(d)}{=}\det\!\Big(\B^\top\B + r\sum_{i=1}^np_i\x_i\x_i^\top\Big),
\end{align*}
where $(a)$ follows from the exchangeability of sequence $\sigma$ (so
that the value of the expectation is the same for any subset $T$),
in $(b)$ we expand the expectation and note that only unique sequences
$\sigma$ will have a non-zero determinant (hence the switch to
subsets and the term $(d-|S|)!$), in $(c)$ we absorb the factors
$r^{d-|S|}$ and $p_i$ into the determinant, and finally $(d)$ is the
classical Cauchy-Binet. It was crucial that we were able to
absorb the subset size $|S|$ into the Poisson series, which allowed
the formula to collapse to a single determinant. This completes the
proof because $\E[\X_\sigma^\top\X_\sigma] =
\E[K]\,\E[\x_{\sigma_1}\x_{\sigma_1}^\top] = r \sum_i p_i\x_i\x_i^\top$.
\end{proof}
In what sense is R-DPP a natural extension of a determinantal point
process? Naively, we might say that setting the matrix $\A$ to an
all-zeros matrix would recover the classical distribution, however
this is not the case because when $\A\!=\!\zero$ only samples of size $d$ or larger will
have non-zero probability.
Instead, we can demonstrate a connection to both DPP and volume
sampling distributions in a different way: we show
that they can be obtained as the limiting distributions of R-DPP when
the regularization and sample size parameter $r$ 
converge to zero (in two separate ways), which is remarkable as the two
distributions in most cases produce vastly different samples.
\begin{theorem}\label{t:limit}
For $\X\in \R^{n\times d}$, ignoring the ordering in the sequences sampled with
$\Vol$, we have
\begin{align*}
  \Vol^{r}\!\Big(\,\X,\ \frac{r}{n}\,\I\,\Big)
  \ &\overset{r\rightarrow 0}{\longrightarrow}\ 
      \DPP(\X)\quad\textnormal{(pointwise)},\\
  \text{whereas}\qquad\Vol^{r}(\X,\zero)
  \ &\overset{r\rightarrow 0}{\longrightarrow}\
      \VS(\X)\quad\textnormal{(pointwise)}.
\end{align*}
\end{theorem}
Theorem \ref{t:limit} suggests that R-DPPs are likely to
be of independent interest as an extension of DPPs. However, it does not
say how to use them algorithmically. To that end, we make a second
observation, which essentially states that DPPs are preserved under
subsampling with R-DPPs. See Appendix \ref{a:r-dpp} for proofs of
Theorems \ref{t:limit} and \ref{t:composition}.
\begin{theorem}\label{t:composition}
For any $\X\in\R^{n\times d}$, $\alpha>0$ and distribution $p$ over
$\{1..n\}$ s.t.~$p_i>0$, let $\Xt$ denote matrix $\X$ with $i$th row rescaled
by $\frac1{\sqrt{\alpha p_i}}$ for every $i\in\{1..n\}$. It follows that for
any $r>0$,
\begin{align*}
  \text{if}\quad\sigmat\sim\Vol_p^r\big(\Xt,\I\big)\ \ \text{and}\ \ 
  S\sim\DPP\big(\Xt_{\sigmat}\big),\quad\text{then}\quad
  \sigmat_S\sim\DPP\Big(\text{\scriptsize$\sqrt{\frac{r}{\alpha}}$}\,\X\Big).
\end{align*}
\end{theorem}

\section{Correctness of Algorithm \ref{alg:main}}\label{s:correctness}
We present the first part of the proof of Theorem
\ref{t:main} by establishing that Algorithm \ref{alg:main} produces a
sample from a determinantal point process. In fact, we will prove the
following more precise claim:
\begin{lemma}\label{l:correctness}
  Given $\X\in\R^{n\times d}$ and a non-zero p.s.d.~matrix
  $\A\in\R^{d\times d}$, Algorithm \ref{alg:main} returns
  $S\sim\DPP(\rho\X)$, with $\rho^2\!=\!\frac{\st}{\hat{s}}$ where $\st\!
  =\!\tr(\A(\I+\A)^{-1})$ and $\hat{s}\!=\!\sum_{i=1}^n\x_i^\top(\I+\A)^{-1}\x_i$.
\end{lemma}
\begin{proof}
The general idea of the proof is to show that the main repeat loop is
implementing an R-DPP, so that we can invoke Theorem
\ref{t:composition}. 
Central to this fact is the choice of numerical factor appearing
in the denominator of Bernoulli sampling probability in line
\ref{line:acc}, which we denote here as 
$C_K=(\frac q {q-\st})^{K+d}\text{e}^{-\st}$. This factor has to depend
on $K$ because otherwise the determinantal term will always dominate it for
large enough $K$. This means that $C_K$ needed to be carefully chosen so that:
\begin{enumerate}
\item the acceptance probability of line \ref{line:acc} is always bounded by 1,
\item we have control over how the presence of $C_K$ changes the distribution of $K$,
\item and $C_K$ is not too large so that we may have a good chance of
  accepting the sample.
\end{enumerate}
We start by showing that the Bernoulli sampling probability in line
\ref{line:acc} is in fact  bounded by 1. We will use the following
simple inequality (proven in Appendix \ref{a:ineq}):
\begin{lemma}\label{l:ineq}
  For any $d\geq 1$, $\epsilon\in[0,1]$, and non-negative integers
  $k,q$ s.t. $q\geq \epsilon d$ we have
  \begin{align*}
    \bigg(\big(1-\epsilon\big) + \frac{\epsilon\, k}{q}\bigg)^d\leq \Big(\frac{q}{q-\epsilon d}\Big)^k\text{e}^{-\epsilon d}.
  \end{align*}
\end{lemma}
Let $\Xt\in\R^{n\times d}$ be the matrix $\X$ where for each
$i\in\{1..n\}$ the $i$th row is rescaled
by
$\sqrt{\white{|}\!\!\!\smash{\text{\fontsize{10}{10}\selectfont$\frac\st{l_i(q-\st)}$}}}$,
with $l_i=\x_i^\top(\I+\A)^{-1}\x_i$ (same as in line \ref{line:acc}
of the algorithm), and let $\sigma=(\sigma_1,\dots,\sigma_K)$.
We use arithmetic-geometric mean
inequality for the eigenvalues of matrix $(\I+\Xt_{\sigma}^\top\Xt_{\sigma})(\I+\A)^{-1}$ and the
fact that $\tr\big((\I+\A)^{-1}\big) = d-\st$, obtaining:
\begin{align*}
  \frac{\det (\I+\Xt_{\sigma}^\top\Xt_{\sigma})}{\det(\I+\A)}
  &= \det\!\big((\I+\Xt_{\sigma}^\top\Xt_{\sigma})(\I+\A)^{-1}\big) \\
  &\leq 
    \bigg(\frac1d\tr\big((\I+\Xt_{\sigma}^\top\Xt_{\sigma})(\I+\A)^{-1}\big)\bigg)^d
  =\bigg(\frac{d\!-\!\st}{d} +
    \frac1{d}\tr\big(\Xt_\sigma^\top\Xt_\sigma(\I+\A)^{-1}\big)\bigg)^d\\
  &=
    \bigg(1\!-\!\frac{\st}{d} +
\frac\st{d(q\!-\!\st)}\sum_{i=1}^K \overbrace{\frac1{l_{\sigma_i}} \x_{\sigma_i}^\top(\I+\A)^{-1}\x_{\sigma_i}\!\!}^{1}\bigg)^d\\
&     = \bigg(1\!-\!\frac{\st}{d}+\frac{\st}{d}\,\frac{K}{q-\st}\bigg)^d
 \leq
     \bigg(1\!-\!\frac{\st}{d}+\frac{\st}{d}\,\frac{K}{q}\bigg)^d\bigg(\frac{q}{q-\st}\bigg)^d\\
&\overset{(*)}{\leq} \bigg(\frac q{q-\st}\bigg)^{K+d} \!\!\text{e}^{-\st}= C_K,
\end{align*}
where $(*)$ follows from Lemma \ref{l:ineq} invoked with
$\epsilon=\frac{\st}{d}$ and $k=K$. Having established the
validity of the rejection sampling in Algorithm \ref{alg:main}, we now
compute the distribution of sample $\sigma$ at the point of exiting
the \textbf{repeat} loop. Denoting $r=q-\st$ as the desired Poisson
mean parameter and $\hat{s} =\sum_il_i= \tr(\X^\top\X(\I+\A)^{-1})$ as
the normalization for the sampling probabilities in line \ref{line:iid},
\begin{align*}
  \Pr(\sigma\, |\, \textit{Acc})\ &\propto
  \overbrace{\frac{\det(\I+\Xt_{\sigma}^\top\Xt_{\sigma})}{(\frac{q}{r})^{K+d}\text{e}^{-\st}\det(\I+\A)}}^{\Pr(\textit{Acc}\,|\,\sigma)}\overbrace{\frac{q^K\text{e}^{-q}}{K!}}^{\Pr(K)}\overbrace{\prod_i\frac{l_{\sigma_i}}{\hat{s}}}^{\Pr(\sigma\,|\,K)}
    \propto\ \det\!\Big(\I+\Xt_\sigma^\top\Xt_\sigma\Big) \frac{r^K\text{e}^{-r}}{K!}\prod_{i=1}^K\frac{l_{\sigma_i}}{\hat{s}},
\end{align*}
where in the above we omitted the normalization for the sake of clarity. Comparing the
obtained unnormalized probability to the one given in
\eqref{eq:poisson-prob}, we conclude that the sample is distributed
according to
$\Vol_l^r\big(\Xt,\,\I\big)$. Note how the
factor $C_K$ interplays with $\Pr(K)$ to ``transform'' the variable from being
$\text{Poisson}(q)$ to $\text{Poisson}(r)$. Invoking
Theorem \ref{t:composition} for the matrix
$\X$, $\alpha=\frac{\hat{s}}{\st}r$ and distribution $\big(\frac{l_1}{\hat{s}},\dots, \frac{l_n}{\hat{s}}\big)$ we conclude that Algorithm \ref{alg:main} returns
$\sigma_{\St}\sim \DPP(\rho \X)$ where $\rho^2=\st/\hat{s}$.
\end{proof}
To bound the rescaling factor $\rho$ we use condition \eqref{eq:cond1}
of Theorem \ref{t:main}
which ensures that $\A=(1\pm\eta)\X^\top\X$, implying that
\begin{align*}
  \st = \tr\big(\A(\I+\A)^{-1}\big)\leq (1+\eta)\,\tr\big(\X^\top\X(\I+\A)^{-1}\big)
  = (1+\eta)\,\hat{s},
\end{align*}
and similarly $\st\geq (1\!-\!\eta)\hat{s}$. We obtain $\rho^2=\st/\hat{s}\in
[1\!-\!\eta,1\!+\!\eta]$ implying that $|\rho-1|\leq \eta$, as claimed in Theorem
\ref{t:main}. Having established this, we use formula
\eqref{eq:bars} to similarly show that: 
\begin{align}\E\big[|\sigma_{\St}|\big]=\tr\big(\rho^2\X^\top\X(\I+\rho^2\X^\top\X)^{-1}\big)\leq 
\Big(1+\frac{\epsilon}{4\bar{s}}\Big)\cdot\E\big[|S|\big]\leq
  \E\big[|S|\big]+\epsilon/4,\label{eq:sbar-bound}
  \end{align}
(lower bound follows identically). The total variation bound (restated below) is proven in Appendix~\ref{a:tv}.
\begin{lemma}\label{l:tv}
There is $C>0$ s.t.~for any matrix $\X$ and $\epsilon\leq 1$, if $|\rho^2-1|\leq
\frac{\epsilon}{4\bar{s}+C\ln9/\epsilon}$,
 where $\bar{s}=\max\{1,\,\E[|S|]\}$ for $S\sim\DPP(\X)$,
then $\DPP(\rho\X)$ is an $\epsilon$-approximation of $\DPP(\X)$.
\end{lemma}

\section{Efficiency of rejection sampling}
\label{s:efficiency}
We complete the proof of Theorem \ref{t:main} by bounding the time
complexity of Algorithm \ref{alg:main} and that of preprocessing.
The key step is to lower bound the probability of exiting the
\textbf{repeat} loop in lines \ref{line:rep1}-\ref{line:rep2}. In the
following lemma we show that if matrix
$\A$ is a sufficiently good approximation of $\X^\top\X$, then the
acceptance probability $\Pr(\textit{Acc}=\text{true})$ is lower bounded by a constant, thus ensuring
that the algorithm will leave the loop after only a few
iterations. 
\begin{lemma}\label{l:efficiency}
If matrix $\A$ satisfies 
$\big(1\!-\!\frac1{4\bar{s}}\big)\X^\top\X\preceq
\A\preceq\big(1\!+\!\frac1{4\bar{s}}\big)\X^\top\X$,
then at the end of each iteration of the {\bf repeat} loop in Algorithm
\ref{alg:main}, we have
$\Pr(\textit{Acc}\!=\!\mathrm{true})\geq \frac16$.
\end{lemma}
\begin{proof}
As in lines \ref{line:poisson} and \ref{line:iid} of Algorithm
\ref{alg:main}, let $\sigma=(\sigma_1,\dots,\sigma_K)\simiid l$ and
$K\sim\mathrm{Poisson}(q)$, where $q=\lceil 2d\st\rceil$. Recall that at the
end of the proof of Lemma \ref{l:correctness} we noted how the
presence of constant $C_K= (\frac{q}{r})^{K+d}\text{e}^{-\st}$, where $r=q-\st$,
appears to transform the distribution of the sample size $K$ into
$\Kt\sim\mathrm{Poisson}(r)$. This can be seen even more clearly as we
compute the acceptance probability after one iteration of the loop:
\begin{align*}
  \E\bigg[\frac{\det(\I+\Xt_\sigma^\top\Xt_\sigma)}{(\frac
  q r)^{K+d}\text{e}^{-\st}\det(\I+\A)}\bigg]
  &=
  \sum_{k=0}^\infty
    \frac{q^k\text{e}^{-q}}{k!}\E\bigg[\frac{
    \det(\I+\Xt_\sigma^\top\Xt_\sigma)}{(\frac q r)^{k+d}\text{e}^{-\st}
    \det(\I+\A)}\,\Big|\,K=k\bigg]\\
  &=  \frac1{(\frac q r)^d}\sum_{k=0}^\infty \frac{r^k\text{e}^{-r}}{k!}\E\bigg[\frac{\det(\I+\Xt_\sigma^\top\Xt_\sigma)}{\det(\I+\A)}\,\Big|\,K=k\bigg]\\
&= \bigg(\frac{q-\st}{q}\bigg)^d\
\frac{\E\Big[\det\!\big(\I+\sum_{i=1}^{\Kt}
\xbt_{\sigma_i}\xbt_{\sigma_i}^\top\big)\Big]}{\det(\I+\A)}\\
&\overset{(*)}{=}\bigg(1 -\frac {\st} {q}\bigg)^d\
\frac{\det\!\big(\I + r\,
\E\big[\xbt_{\sigma_1}\xbt_{\sigma_1}\big]\big)}{\det(\I+\A)},
\end{align*}
where $(*)$ follows from Lemma \ref{l:cb} applied to $\Xt$,
distribution $l$ and sample size $\Kt$. Bernoulli's
inequality shows that $(1\!-\!\frac{\st}{q})^d\geq
1\!-\!\frac{d\st}{q}\geq \frac12$. Furthermore, it is easy to verify
that $r\,\E[\xbt_{\sigma_1}\xbt_{\sigma_1}^\top]=\st\ 
\E[\frac1{l_{\sigma_1}}\x_{\sigma_1}\x_{\sigma_1}^\top] =
\rho^2\X^\top\X$, where $\rho^2=\st/\hat{s}$. 
To lower bound the ratio of determinants we use the following lemma
shown in Appendix \ref{a:tv}.
\begin{lemma}\label{l:det-bound}
  For p.s.d.~matrices $\B$, $\C$ such that $(1-\gamma)\C\preceq
  \B\preceq (1+\gamma)\C$ with $\gamma\in(0,1)$,
  \begin{align*}
\text{e}^{-\frac{\gamma}{1-\gamma}s}\det(\I+\C)\leq\det(\I+\B)\leq
    \text{e}^{\gamma s}\det(\I+\C),\quad\text{where}\quad s=\tr\big(\C(\I+\C)^{-1}\big).
  \end{align*}
\end{lemma}
Setting $\B=\A$ and $\C=\rho^2\X^\top\X$, we have $\B\preceq
(1+\eta)\X^\top\X\preceq \frac{1+\eta}{1-\eta}\C$ and similar lower bound follows, so applying Lemma \ref{l:det-bound}
with $\gamma=\frac{2\eta}{1-\eta}$ and $s=\tr(\C(\I+\C)^{-1})\leq
\frac54\bar{s}$ (see \eqref{eq:sbar-bound}):
\begin{align*}
\frac{\det(\I+\rho^2\X^\top\X)}{\det(\I+\A)}\geq
  \text{e}^{-\frac{2\eta}{1-\eta}\frac54\bar{s}}\geq
  \text{e}^{-\frac1{2\bar{s}}\frac43\frac54\bar{s}}= \text{e}^{-\frac56}\geq
  \frac13,
\end{align*}
where we used the fact that $\eta \leq \frac1{4\bar{s}}$. Thus
the acceptance probability is at least $\frac12\cdot\frac13=\frac16$.
\end{proof}
The remaining steps in proving the time complexity bound for
Algorithm \ref{alg:main} are standard, 
and so they were relegated to Appendix \ref{a:time} along with the
proof of Proposition \ref{p:preprocessing}.

\bibliography{pap}

\appendix
\newpage

\section{Properties of regularized determinantal point processes}
\label{a:r-dpp}
We give the proofs omitted from Section \ref{s:r-dpp}.
\paragraph{Theorem \ref{t:limit}}\hspace{-2.5mm}
\textit{For $\X\in \R^{n\times d}$, ignoring the ordering in the sequences sampled with
$\Vol$, we have}
\begin{align}
  \Vol^{r}\!\Big(\,\X,\ \frac{r}{n}\,\I\,\Big)
  \ &\overset{r\rightarrow 0}{\longrightarrow}\ 
      \DPP(\X)\quad\textnormal{(pointwise)},\label{eq:lim-dpp}\\
  \textit{whereas}\qquad\Vol^{r}(\X,\zero)
  \ &\overset{r\rightarrow 0}{\longrightarrow}\
      \VS(\X)\quad\textnormal{(pointwise)}.\label{eq:lim-vs}
\end{align}
\begin{proof}
To show \eqref{eq:lim-dpp} we use a fact which is a simple consequence of the
Sylvester's theorem, namely that 
  $\det(\frac{r}{n}\I+\X_S^\top\X_S) =
  (\frac{r}{n})^{d-k}\det(\frac{r}{n}\I+\X_S\X_S^\top)$ for a set $S$ of size $k$. It follows that
  \begin{align*}
    \Pr(\sigmat\!=\! S) &= k! \,\frac{\det(\frac{r}{n}\I +
    \X_S^\top\X_S)}{\det(\frac{r}{n}\I + \frac{r}{n}\X^\top\X)}\,
    \frac{r^k\text{e}^{-r}}{k!\,n^k} =
    \frac{(\frac{r}{n})^{d-k}\det(\frac{r}{n}\I +
    \X_S\X_S^\top)}{(\frac{r}{n})^d\det(\I+\X^\top\X)}
    \Big(\frac{r}{n}\Big)^k\text{e}^{-r}\\
&=\frac{\det(\frac{r}{n}\I+\X_S\X_S^\top)}{\det(\I+\X\X^\top)}
                          \text{e}^{-r}
      \ \overset{r\rightarrow 0}{\longrightarrow}\
      \frac{\det(\X_S\X_S^\top)}{\det(\I+\X\X^\top)},
  \end{align*}
  where $\sigmat=S$ should be interpreted as if $\sigmat$ was an
  unordered multiset. Next, we prove \eqref{eq:lim-vs}:
  \begin{align*}
    \Pr(\sigmat\!=\! S) = k! \,\frac{\det(\X_S^\top\X_S)}
    {\det(\frac{r}{n}\X^\top\X)}\,
    \frac{r^k\text{e}^{-r}}{k!\,n^k}
    =\frac{\det(\X_S^\top\X_S)}{\det(\X^\top\X)}
    \Big(\frac{r}{n}\Big)^{k-d}\text{e}^{-r}
    \ \overset{r\rightarrow 0}{\longrightarrow}\
    \one_{[k=d]}\frac{\det(\X_S^\top\X_S)}{\det(\X^\top\X)},
  \end{align*}
  because $\det(\X_S^\top\X_S)=0$ whenever $k<d$.
\end{proof}
\paragraph{Theorem \ref{t:composition}}\hspace{-2.5mm}
\textit{For any $\X\in\R^{n\times d}$, $\alpha>0$ and distribution $p$ over
$\{1..n\}$ s.t.~$p_i>0$, let $\Xt$ denote matrix $\X$ with $i$th row rescaled
by $\frac1{\sqrt{\alpha p_i}}$ for every $i\in\{1..n\}$. It follows that for
any $r>0$,}
\begin{align*}
  \textit{if}\quad\sigmat\sim\Vol_p^r\big(\Xt,\I\big)\ \ \textit{and}\ \ 
  S\sim\DPP\big(\Xt_{\sigmat}\big),\quad\textit{then}\quad
  \sigmat_S\sim\DPP\Big(\text{\scriptsize$\sqrt{\frac{r}{\alpha}}$}\,\X\Big).
\end{align*}
\begin{proof}
Using the law of total
probability we compute the probability of sampling set $T$ of size $t$:
  \begin{align*}
    \Pr(\sigmat_S\!=\!T)
    &=\sum_{\sigmat}\Pr(\sigmat_S\!=\!T\,|\,\sigmat)\Pr(\sigmat)
    \\[-6mm]
   &\overset{(a)}{=}\sum_{k=t}^\infty\sum_{S\subseteq [k]}\sum_{\sigmat:\,|\sigmat|=k}\one_{[\sigmat_S=T]} \overbrace{\frac{\det(\Xt_{\sigmat_S}\Xt_{\sigmat_S}^\top)}{\cancel{\det(\I+\Xt_{\sigmat}\Xt_{\sigmat}^\top)}}}^{\Pr(S\,|\,\sigmat)}\overbrace{
    \frac{\cancel{\det(\I+\Xt_{\sigmat}^\top\Xt_{\sigmat})}}{\det\!\big(\I +
     r\,\sum_i\frac{p_i}{\alpha p_i}\x_i\x_i^\top\big)}\,\frac{r^k\text{e}^{-r}}{k!}
     \prod_{i=1}^kp_{\sigmat_i}}^{\Pr(\sigmat)}\\
    &\overset{(b)}{=}\sum_{k=t}^\infty {k\choose t}t!\frac{\det(\Xt_T\Xt_T^\top)}{\det(\I
      +
      \frac{r}{\alpha}\X^\top\X)}\frac{r^k\text{e}^{-r}}{k!}\prod_{i\in
      T}p_{i}\\
    &=\frac{\det(\Xt_T\Xt_T^\top)}{\det(\I+\frac{r}{\alpha}\X^\top\X)}
      \bigg(\prod_{i\in T}p_i\bigg)\ \sum_{k=t}^\infty \frac{k!}{(k-t)!}\,
      \frac{r^k\text{e}^{-r}}{k!}\\
    &=\frac{\det(\frac{r}{\alpha}\X_T\X_T^\top) \big(\prod_{i\in
      T}\frac1{p_i}\big)}{\det(\I+\frac{r}{\alpha}\X^\top\X)}
      \bigg(\prod_{i\in T}p_i\bigg)\ 
      \sum_{k=t}^\infty \frac{r^{k-t}\text{e}^{-r}}{(k-t)!}
      \ =\
      \frac{\det(\frac{r}{\alpha}\X_T\X_T^\top)}
      {\det(\I+\frac{r}{\alpha}\X\X^\top)},
  \end{align*}
  where the cancellation in $(a)$ follows from Sylvester's Theorem,
  and in $(b)$ we use the exchangeability of sequence $\sigmat$ to
  observe that for any subset $S$ of size $t$ the value of the
  proceeding sum is the same (the factor ${k\choose t}$ counts the
  number of such subsets $S$ and $t!$ counts the number of sequences
  of length $t$ that correspond to set $T$).
\end{proof}

\section{Proof of Lemma \ref{l:ineq}}
\label{a:ineq}
We present the omitted proof of the
inequality from Lemma~\ref{l:ineq}, stated here again.


\paragraph{Lemma \ref{l:ineq}}\textit{For any $d\geq 1$, $\epsilon\in[0,1]$, and non-negative integers
  $k,q$ s.t.~$q\geq \epsilon d$ we have
  \begin{align}
    \bigg(\big(1-\epsilon\big) + \frac{\epsilon\, k}{q}\bigg)^d\leq \Big(\frac{q}{q-\epsilon d}\Big)^k\text{e}^{-\epsilon d}.\label{eq:ineq}
  \end{align}
}
\begin{proof}
We start with the case of $d=1$. Denote the left hand side of
\eqref{eq:ineq} as $L_k$ and the right hand side as $R_k$. If $k=q$ then
\begin{align*}
 R_q= \bigg(\frac{q}{q-\epsilon }\bigg)^{q}\text{e}^{-\epsilon}
  =
  \frac{\text{e}^{-\epsilon}}{\big(1-\frac{\epsilon}{q}\big)^{q}}\geq
  1= 1-\epsilon + \epsilon \frac{q}{q}= L_q.
\end{align*}
Let us now consider the multiplicative change in
$L_k$ as we increase or decrease $k$ by one:
\begin{align*}
  \frac{L_k}{L_{k+1}} = \frac{L_{k+1} - \frac{\epsilon}{q}}{L_{k+1}} =
  1 - \frac{\epsilon}{q\,L_{k+1}}\ 
  \begin{cases}
   \ \leq 1-\frac{\epsilon}{q} \qquad &
   \text{for }k\leq q-1 \text{ because }L_{k+1}\leq 1,\\
     \  \geq 1-\frac{\epsilon}{q} \qquad &
      \text{for }k\geq q \text{ because }L_{k+1}\geq 1.
    \end{cases}
\end{align*}
Observe that $\frac{R_k}{R_{k+1}}=1- \frac\epsilon q$ for any $k$, so
by induction over decreasing $k\leq q-1$,
\begin{align*}
  L_k = L_{k+1}\frac{L_k}{L_{k+1}}\leq R_{k+1}\Big(1-\frac\epsilon
  q\Big) = R_{k+1}\,\frac{R_k}{R_{k+1}} = R_{k+1},
\end{align*}
and for increasing $k\geq q$ similar induction shows that $L_{k+1} =
L_{k}\frac{L_{k+1}}{L_k}\leq R_k\frac{R_{k+1}}{R_k}=R_{k+1}$. Finally, we
use the case $d=1$ to show the inequality for arbitrary $d\geq 1$:
\begin{align*}
  \bigg(1-\epsilon + \epsilon\frac k q\bigg)^d
  &\overset{(a)}{\leq}
  \bigg(\Big(1-\frac \epsilon q\Big)^{-k}\text{e}^{-\epsilon}\bigg)^d
  =\bigg(\!\frac1{(1-\frac\epsilon q)^d}\!\bigg)^k\text{e}^{-\epsilon
                                                    d}\\
  &\overset{(b)}{\leq}
  \bigg(\frac1{1-\frac{\epsilon d} q}\bigg)^k\text{e}^{-\epsilon d}
  =\Big(\frac q{q-\epsilon d}\Big)^k\text{e}^{-\epsilon d}\!,
\end{align*}
where $(a)$ follows from \eqref{eq:ineq} applied for $d=1$ and
$(b)$ is Bernoulli's inequality.
\end{proof}

\section{Total variation bound for Algorithm \ref{alg:main}}
\label{a:tv}
We start by showing an approximation lemma about determinants which
was earlier given in Section \ref{s:correctness} (we restate
it here).
\paragraph{Lemma \ref{l:det-bound}}\hspace{-2.5mm}\textit{
  For p.s.d.~matrices $\B$, $\C$ such that $(1-\gamma)\C\preceq
  \B\preceq (1+\gamma)\C$ with $\gamma\in(0,1)$,}
  \begin{align*}
\text{e}^{-\frac{\gamma}{1-\gamma}s}\det(\I+\C)\leq\det(\I+\B)\leq
    \text{e}^{\gamma s}\det(\I+\C),\quad\textit{where}\quad s=\tr\big(\C(\I+\C)^{-1}\big).
  \end{align*}
\begin{proof}
Let $\lambda_1,\dots,\lambda_d$ denote the eigenvalues of
$\C(\I+\C)^{-1}$. Then,
  \begin{align*}
  \frac{\det(\I+\B)}{\det(\I+\C)}
  &= \det\!\big(\I +
  (\B-\C)(\I+\C)^{-1}\big)\\
    &\leq \det\!\big(\I+\gamma \,\C(\I+\C)^{-1}\big)\\
    &= \prod_{i=1}^d(1+\gamma\lambda_i)
  \leq \prod_{i=1}^d\text{e}^{\gamma\lambda_i} = \text{e}^{\gamma\,\tr(\C(\I+\C)^{-1})},
\end{align*}
which gives the upper bound. Similarly, we have
\begin{align*}
  \frac{\det(\I+\C)}{\det(\I+\B)}
  &=\det\!\big(\I+(\C-\B)(\I+\B)^{-1}\big)\\
  &\leq \det\!\big(\I+\gamma\,\C(\I+(1-\gamma)\C)^{-1}\big)\\
  &\leq \det\!\Big(\I+\frac{\gamma}{1-\gamma}\,\C(\I+\C)^{-1}\Big)
    \leq \text{e}^{\frac{\gamma}{1-\gamma}\,\tr(\C(\I+\C)^{-1})},
\end{align*}
so by inverting both sides we obtain the lower bound.
\end{proof}
We are ready to prove that $\DPP(\rho\X)$ is an approximation of
$\DPP(\X)$ in terms of total variation distance (restated here).
\paragraph{Lemma \ref{l:tv}}\hspace{-2.5mm}\textit{
There is $C>0$ s.t.~for any matrix $\X$ and $\epsilon\leq 1$, if $|\rho^2-1|\leq
\frac{\epsilon}{4\bar{s}+C\ln9/\epsilon}$,
 where $\bar{s}=\max\{1,\,\E[|S|]\}$ for $S\sim\DPP(\X)$,
then $\DPP(\rho\X)$ is an $\epsilon$-approximation of $\DPP(\X)$.}
\vspace{4mm}

\begin{proof}
The larger the size of subset $S$ the harder it is to control its
approximate probability because it is defined via the determinant of a
larger matrix. To overcome this we use the following
standard concentration bound for determinantal point processes which shows that
the probability of sampling a large subset is negligibly small. For
simplicity, we state only a special case of the cited result. 
\begin{lemma}[based on \citeauthor{dpp-concentration},
  \citeyear{dpp-concentration}, Theorem 3.5]
  \label{l:concent}
 Given any $\X$, if $S\sim\DPP(\X)$, then for
 any $a>0$ we have:
 \begin{align*}
   \Pr\big( |S| - \E[|S|] \geq a\big)\leq 3\exp\bigg(-\frac{a^2}{16(a+2\,\E[|S|])}\bigg).
 \end{align*}
\end{lemma}
In \eqref{eq:sbar-bound} we showed that the expected subset size
$|S|$ for both $\DPP(\X)$ and $\DPP(\rho\X)$ 
is bounded by $\bar{s}+\epsilon/4\leq (1+\frac14)\bar{s}$, so setting $a = (3/4)\bar{s} + 80\ln
9/\epsilon$ in Lemma \ref{l:concent} (applied to either distribution),
\begin{align*}
  \Pr\Big(|S| \geq 2\bar{s}+80\ln 9/\epsilon\Big)&\leq
  \Pr\Big(|S|-\E[|S|]\geq a\Big)\\
  &\leq
  3\exp\bigg(-\frac{((3/4)\bar{s}+80\ln9/\epsilon)^2}{16((13/4)\bar{s}+80\ln9/\epsilon)}\bigg)\\
  &\leq
  3\exp\bigg(-\frac{(3/4)\bar{s}+80\ln9/\epsilon}{5\cdot 16}\bigg)\leq \frac{\epsilon}{3}.
\end{align*}
So letting $k=2\bar{s}+80\ln9/\epsilon$ and $C=160$ we have
$|\rho^2-1|\leq \frac{\epsilon}{2k}$.  We
can now bound the total variation distance between the distributions
$p(S)=\frac{\det(\X_S\X_S^\top)}{\det(\I+\X\X^\top)}$ and
$q(S)=\frac{\det(\rho^2\X_S\X_S^\top)}{\det(\I+\rho^2\X\X^\top)}$.
\begin{align*}
  \frac12\sum_{S\subseteq\{1..n\}}\!\!|p(S)-q(S)|
  &\overset{(a)}{\leq} \frac{\epsilon}{3} + 
    \frac12\sum_{S:|S|\leq k}
    \bigg|\frac{\det(\X_S\X_S^\top)}{\det(\I+\X\X^\top)} -
    \frac{\rho^{2|S|}\det(\X_S\X_S^\top)}{\det(\I+\rho^2\X\X^\top)}\bigg|
\\ 
  &=\frac{\epsilon}{3} + \frac12\sum_{S:|S|\leq k}
    p(S)\cdot\bigg|1-\rho^{2|S|}\,\frac{\det(\I+\X\X^\top)}
    {\det(\I+\rho^2\X\X^\top)}\bigg|\\
  &\overset{(b)}{\leq}\frac{\epsilon}{3} + \frac{(1+\frac\epsilon{2k})^k\,
    \text{e}^{\frac\epsilon{4\bar{s}}\cdot\bar{s}}-1}{2} \sum_{S:|S|\leq k}
    p(S)\\
  &\leq \frac{\epsilon}{3} + \frac{\text{e}^{(3/4)\epsilon}-1}{2}\leq \epsilon,
\end{align*}
where $(a)$ uses Lemma \ref{l:concent} and in $(b)$ we used
Lemma \ref{l:det-bound} and the fact that $|\rho^2-1|\leq \frac\epsilon{2k}$.
\end{proof}

\section{Time complexity proofs}
\label{a:time}
In this Section we show the runtime bounds for both Algorithm
\ref{alg:main} and the preprocessing.

\subsection{Sampling cost (proof of Theorem \ref{t:main}, part 3)}
Lemma \ref{l:efficiency} implies that with probability at least
$1-\delta$ Algorithm \ref{alg:main} will perform
$\ln(\frac1\delta)/\ln(\frac65 )$
iterations of the loop. We next analyze the cost of one such
iteration. Note that in line \ref{line:iid} we are supposed to sample exactly
from the distribution $l=(l_1,\dots,l_n)$ even though we are only
given its approximation $\tilde{l}$ with condition \eqref{eq:cond2}
stating that $\frac12 l_i\leq\tilde{l}_i\leq\frac32 l_i$ for all
$i\in\{1..n\}$. We can do this via simple rejection sampling performed
for each $t\in\{1..K\}$:
\begin{center}
  Sample $i\sim \tilde{l}$,
  \quad $a\sim \mathrm{Bernoulli}\big(\frac{l_i}{2\tilde{l}_i}\big)$,
  \quad if $a=\mathrm{true}$, then $\sigma_t=i$, else repeat.
\end{center}
From the approximation guarantee it follows that the Bernoulli
probability is bounded by 1 and never less than $\frac13$.
One such probability requires computing
$l_i=\x_i^\top(\I+\A)^{-1}\x_i$ which takes $O(d^2)$ if the matrix inverse
$(\I+\A)^{-1}$ was precomputed (in time $O(d^3)$). How many times will
we need to compute $l_i$? Let $m$ denote the total number of i.i.d.~samples
from $l$ needed throughout the algorithm, i.e.~the sum of all of the Poisson
variables $K$. Conditioned on $m$, with
probability at least $1-\delta$ the total cost of sampling from line
\ref{line:iid} over the course of the algorithm is
$O(md^2\ln(\frac1\delta))$. The total cost of computing the
determinants from
line \ref{line:acc} as well as the cost of sampling from
$\DPP(\Xt_{\sigma})$ is $O(md^2)$ so they do not add to the
asymptotic 
runtime. Since the number of iterations of the \textbf{repeat}
loop is w.p. $\geq 1-\delta$ bounded by
$c=\lceil\ln(\frac1\delta)/\ln(\frac65)\rceil$, let variable $\widehat{m}$ be the sum
of $c$ independent copies of $K$. Then
$\widehat{m}\sim\mathrm{Poisson}(c\,q)$
and a Poisson tail bound \citep{poisson-tail} for any $\alpha>0$ yields 
\begin{align*}
  \Pr\big(\widehat{m}\geq c\,(q+\alpha)\big) \leq
  \text{e}^{-\frac{(c\,\alpha)^2}{c\,\alpha+c\,q}}
  \leq\text{e}^{2\ln(\delta)\frac{\alpha^2}{\alpha+q}}
  = \delta^{\frac{2\alpha^2}{\alpha+q}},
\end{align*}
which is less than $\delta$ for $\alpha=q$. Thus, with
probability at least $1-2\delta$ we have $m\leq\widehat{m}\leq
2cq = O(d(\st+1)\log\frac1\delta)$, and
the overall time complexity of Algorithm \ref{alg:main} becomes 
$O(d^3(\st+1)\log^2\frac1\delta)$. Since $\st=\tr(\A(\I+\A)^{-1}) \leq
\frac{1+1/(4\bar{s})}{1-1/(4\bar{s})}\cdot\bar{s}\leq \bar{s}+1$, we obtain
the bound in Theorem \ref{t:main}.

\subsection{Preprocessing cost (proof of Proposition \ref{p:preprocessing})}
\label{s:preprocessing}

Matrix $\A\in\R^{d\times d}$ and estimate
distribution $\tilde{l}$ satisfying approximation guarantees
\eqref{eq:cond1} and \eqref{eq:cond2} can be computed efficiently using
standard sketching and/or sampling techniques. Here,
we outline the basic steps needed to obtain this, and discuss the
time complexity achievable for each step:
\begin{enumerate}
  \item Compute $\frac12$-approximate leverage score distribution
    $p=(p_1,\dots,p_n)$, i.e.~such that the probabilities
    $p_i$ satify (here ``$(\cdot)^+$'' is the Moore-Penrose pseudo-inverse):
    \begin{align*}
p_i \geq \frac{\x_i^\top(\X^\top\X)^{+}\x_i}{2\,\rank(\X)}.
    \end{align*}
    \item Sample $r(\eta)$ row indices
      $\sigma_1,\dots,\sigma_{r(\eta)}\simiid p$, so that with
      high probability
      \begin{align}
        (1-\eta)\X^\top\X\preceq \overbrace{\frac1{r(\eta)}\sum_{i}
        \frac1{p_{\sigma_i}}\x_{\sigma_i}\x_{\sigma_i}^\top}^{\A}
        \preceq (1 + \eta)\X^\top\X.\label{eq:concentration}
      \end{align}
      \item Having found matrix $\A$ we compute the approximate
        distribution $\tilde{l}$ satisfying
    \begin{align*}
      \frac12\x_i^\top(\I+\A)^{-1}\x_i\leq \tilde{l}_i \leq \frac32
      \x_i^\top(\I+\A)^{-1}\x_i,
    \end{align*}
which is similar to the leverage scores, except with matrix $\X^\top\X$ replaced by $\I+\A$.
  \end{enumerate}
Step 1 can be performed in time $O(\nnz(\X)\log n + d^3\log^2d +
d^2\log n)$ by employing the sparse subspace embedding technique developed by
\cite{cw-sparse}. Similar running times are offered by embeddings proposed by
\cite{nn-sparse,mm-sparse}.
Step 2
is an application of the standard matrix concentration bounds due to
\cite{matrix-tail-bounds}, which show that it suffices to sample
$r(\eta) = O(d\eta^{-2}\log d)$ rows from distribution
$p$ to satisfy \eqref{eq:concentration} with high probability. The
computation of matrix $\A$ then takes $O(r(\eta)d^2) =
O(d^3\eta^{-2}\log d)$.
Finally, step 3 is very similar to step 1,
except we are estimating ridge leverage score type values. A standard
approach of doing this is to observe that $\x_i^\top(\I+\A)^{-1}\x_i$
is the squared norm of the $i$th row in matrix
$\X(\I+\A)^{-\frac12}$. All of the row norms of this matrix can be
estimated in time $O(\nnz(\X)\log n + d^3 + d^2\log n)$ using the
Johnson-Lindenstraus tranform as described by
\cite{fast-leverage-scores}. Thus, the overall time complexity is
$O(\nnz(\X)\log n + d^3\eta^{-2}\log d + d^2\log n)$, where recall
that $\eta^{-1}=O((\bar{s}+\log1/\epsilon)/\epsilon)$. Note that the
term $d^2\log n$
can be omitted from the time complexity, because if $\log n=\Omega(d)$,
then $\nnz(\X)\log n = \Omega(\nnz(\X)\,d)$ and we can compute
both $\X^\top\X$ and $\l$ exactly in time $O(\nnz(\X)\,d+d^3)$. Also,
our procedure requires a 
constant factor estimate of $\bar{s}$ to compute $\eta$. In fact, we
can first use $\eta_0=\frac12$ and let $\max\{1,\,\tr(\A(\I+\A)^{-1})\}$ be
the estimate for $\bar{s}$ and then perform a second more accurate estimation.



\end{document}